\documentclass{article}
\usepackage{amsmath,graphicx,mlspconf}

\usepackage{microtype}
\usepackage{graphicx}
\usepackage{subfigure}
\usepackage{booktabs} 

\usepackage{amsmath}
\usepackage{amssymb}
\usepackage{mathtools}
\usepackage{amsthm}
\usepackage{multirow}

\usepackage{hyperref}
\usepackage[capitalize,noabbrev]{cleveref}
\usepackage{xcolor}

\theoremstyle{plain}
\newtheorem{theorem}{Theorem}[section]
\newtheorem{proposition}[theorem]{Proposition}

\theoremstyle{definition}
\newtheorem{definition}[theorem]{Definition}

\theoremstyle{remark}

\usepackage{algorithm}
\usepackage{algorithmic}
\usepackage{forloop}

\copyrightnotice{U.S.\ Government work not protected by U.S.\ copyright}

\copyrightnotice{979-8-3503-2411-2/23/\$31.00 {\copyright}2023 Crown}

\copyrightnotice{979-8-3503-2411-2/23/\$31.00 {\copyright}2023 European Union}

\copyrightnotice{979-8-3503-2411-2/23/\$31.00 {\copyright}2023 IEEE}

\toappear{2023 IEEE International Workshop on Machine Learning for Signal Processing, Sept.\ 17--20, 2023, Rome, Italy}

\title{Geodesic Sinkhorn for Fast and Accurate Optimal Transport on Manifolds}
\name{%
   \parbox{\textwidth}{\centering{}Guillaume Huguet$^{1,\star}$%
   \qquad Alexander Tong$^{1,\star}$%
   \qquad María Ramos Zapatero$^2$\\%
   \qquad {Christopher J. Tape}$^2$%
   \qquad {Guy Wolf}$^1$%
   \qquad {Smita Krishnaswamy}$^3$}\thanks{$^*$Equal contribution. This research was partially funded by ESP \textit{Mérite} [G.H.], NSERC Discovery grant 03267 [G.W.], Canada CIFAR AI Chair [G.W.], NIH grant R01GM135929 [G.W., S.K.], Cancer Research UK (C60693 / A23783) [C.J.T.], the Rosetrees Trust (M872 / A2292) [C.J.T.], and the Yale-UCL Collaborative Student Exchange Programme. Correspondence to: guy.wolf@umontreal.ca, smita.krishnaswamy@yale.edu}%
}

\address{%
   $^1$ Université de Montréal; Mila - Quebec AI Institute \\
   $^2$ University College London Cancer Institute 
   $^3$ Yale University 
}

\usepackage{amsmath,amsfonts,bm}
\usepackage{soul} 

\def\eqref#1{equation~\ref{#1}}

\def\1{\bm{1}}

\def\vone{{\bm{1}}}
\def\vmu{{\bm{\mu}}}
\def\vnu{{\bm{\nu}}} 

\def\va{{\bm{a}}}

\def\vf{{\bm{f}}}

\def\vv{{\bm{v}}}
\def\vw{{\bm{w}}}

\def\mA{{\bm{A}}}

\def\mD{{\bm{D}}}

\def\mH{{\bm{H}}}

\def\mL{{\bm{L}}}

\DeclareMathAlphabet{\mathsfit}{\encodingdefault}{\sfdefault}{m}{sl}
\SetMathAlphabet{\mathsfit}{bold}{\encodingdefault}{\sfdefault}{bx}{n}

\def\gC{{\mathcal{C}}}

\def\gG{{\mathcal{G}}}

\def\gP{{\mathcal{P}}}

\def\gT{{\mathcal{T}}}

\def\gX{{\mathcal{X}}}

\newcommand{\KL}{D_{\mathrm{KL}}}

\DeclareMathOperator*{\argmin}{arg\,min}

\newcommand\first[1]{{\bf{#1}}}

\DeclareUnicodeCharacter{0301}{\'{e}}
\usepackage[style=ieee]{biblatex}

\begin{document}

\maketitle

\begin{abstract}
Efficient computation of optimal transport distance between distributions is of growing importance in data science. Sinkhorn-based methods are currently the state-of-the-art for such computations, but require $O(n^2)$ computations. In addition, Sinkhorn-based methods commonly use an Euclidean ground distance between datapoints. However, with the prevalence of manifold structured scientific data, it is often desirable to consider geodesic ground distance. Here, we tackle both issues by proposing Geodesic Sinkhorn---based on diffusing a heat kernel on a manifold graph. Notably, Geodesic Sinkhorn requires only $O(n\log n)$ computation, as we approximate the heat kernel with Chebyshev polynomials based on the sparse graph Laplacian. We apply our method to the computation of barycenters of several distributions of high dimensional single cell data from patient samples undergoing chemotherapy. In particular, we define the barycentric distance as the distance between two such barycenters. Using this definition, we identify an optimal transport distance and path associated with the effect of treatment on cellular data.

\end{abstract}

\section{Introduction}
Optimal Transport (OT) distances or Wasserstein distances are computed by lifting ground distances between points to distances between measures. This distance is computed relative to a ground distance on the support of the distributions, making it more informative than distances based only on a pointwise comparison of the densities. However, to compute the Wasserstein, one needs to find the optimal transport plan from the source distribution to a target distribution; this is a linear programming problem requiring $O(n^3 \log n)$ for discrete distributions of size $n$~\cite{peyre_computational_2020}. 

An efficient modification of the optimal transport problem is to consider entropy-regularized transportation. This formulation is solved with the Sinkhorn algorithm~\cite{sinkhorn1967concerning} by iteratively rescaling a Gaussian kernel based on the distance matrix. It is equivalent to the Schrödinger Bridge problem, for which similar algorithms were developed~\cite{fortet1940resolution,kullback1968probability,knight2013fast}. In the discrete case, it requires $O(n^2)$ for distributions of size $n$, since it relies on matrix-vector products. Furthermore, this formulation allows for fast computation of the discrete barycenter with fixed support (the average distributions w.r.t.\ the Sinkhorn distance). An important drawback of the Sinkhorn algorithm is the necessity of storing and multiplying the pairwise distance matrix with a vector. 

Additionally, the ground distance is commonly chosen as the Euclidean distance. The Euclidean distance is often suboptimal for high-dimensional datasets over larger distances according to the manifold hypothesis,
which says observations lie near a low dimensional (curved) manifold in high dimensional space~\cite{moon_manifold_2018}. For higher dimensional datasets assumed to be sampled from a lower dimensional manifold, using a distance closer to the manifold for OT has shown promising results~\cite{huguet2022manifold,solomon2015convolutional,tong2022embedding,tong2021diffusion}. 

In this work, we present Geodesic Sinkhorn\footnote{\url{https://github.com/KrishnaswamyLab/GeoSinkhorn}}; a Sinkhorn-based method for fast optimal transport with a heat-geodesic ground distance. Our method is based on the geometry of the dataset constructed from a common graph and uses the heat kernel on the graph to defined a heat-geodesic distance. Key to this approach, we will never need to construct or operate on an $n \times n$ distance matrix, and we will only use the sparse Laplacian matrix and sparse matrix-vector products. For sparse graphs, this can be used for $O(n \log n)$ computation of the Sinkhorn distance with a manifold ground distance, improving on the $O(n^2)$ implementations based on dense matrices.

Increasing the state-of-the art efficiency in Sinkhorn computation opens us up to being able to perform complex operations on large groups of datasets. In particular, we consider interpolating between datasets and show that using our heat-geodesic distance improves the interpolation accuracy compared to OT with Euclidean distance. The barycenter corresponds to the average distribution of a set of distributions. Our method allows for finer-grained barycenters on a data manifold, which motivates us to define a novel notion of dissimilarity between families of distributions called {\em barycentric distance}.

We apply the barycentric distance to single cell data from patient-derived cancer organoids (PDOs) to assess the effect of treatments (such as drugs and chemotherapy). Here we have one set of PDOs from control conditions, and another set that are treated. The treatment effect is thus the distance between these barycenters. In addition, we use Geodesic Sinkhorn's barycenter to compare the effect from one family of distributions to another.

Our main contributions include: 
    (1) A new method for computing optimal transport distances on a manifold called Geodesic Sinkhorn, which is highly efficient in time and memory.
    (2) Defining the barycentric distance; a novel distance between \emph{families} of distributions, and showing its utility in deriving treatment effect from control and treated patient samples.

\section{Related Work}

Geodesic Sinkhorn is related to prior work linking the entropy-regularized optimal transport problem triangular mesh with the heat operator~\cite{crane_geodesics_2013,solomon2015convolutional}, but using different graph filtering techniques. These approaches approximate the application of the heat kernel to a vector by discretizing the heat equation and solving systems of linear equations. This technique was used in different contexts, either with the cotangent Laplacian~\cite{solomon2015convolutional} or to learn a ground metric~\cite{heitz2021ground}. Solving these systems for each Sinkhorn iteration can be done efficiently with a sparse Cholesky decomposition. However, this method's efficiency depends mainly on the efficiency of the Cholesky decomposition which can be slow depending on the sparsity pattern is $O(n^3)$ for an $n \times n$ matrix, and necessitates solving $2K$ systems of linear equations per Sinkhorn iteration, where $K$ is the number of sub-steps in the backward Euler discretization. 



\section{Preliminaries}

In this section, we start by reviewing the basics of OT and the Wasserstein distance, as well the Sinkhorn distance. Then we review two notions fundamental to our method; the heat equation on a graph and the Chebyshev approximation of the heat kernel.


\subsection{Wasserstein Distance}
In the following, we assume that all distributions admit a density or a probability mass function, and we use the same notation for both. Let $\mu, \nu$ be two probability distributions on a measurable space $\gX\subseteq \mathbb{R}^d$ with metric $d(\cdot,\cdot)$, let $\Pi(\mu, \nu)$ be the set of joint probability distributions $\pi$ on the space $\gX\times\gX$ where, for any measurable subset $\omega \subset \gX$, $\pi(\omega \times \gX) = \mu(\omega)$ and $\pi(\gX \times \omega) = \nu(\omega)$. The $p$-Wasserstein distance is defined as:
\vspace{-2mm}
\begin{equation}\label{eq: wasserstein}
    W_p(\mu,\nu) := \left (\inf_{\pi \in \Pi(\mu, \nu)} \int_{\gX^2} d(x,y)^p \mathrm{d}\pi(x, y) \right)^{1/p}.
\end{equation}
In the following, we consider $p = 2$. An exact algorithm based on linear programming can solve this problem in $O(n^3 \log n)$ time for discrete distributions of size $n$.

\subsection{Sinkhorn Distances}
The \textit{Kullback-Leibler} (KL) divergence between $\pi$ and some strictly positive $K$ on $\gX\times\gX$ is defined as
\begin{equation}\label{eq:kl}
    \KL(\pi | K) := \int_{\gX^2} \left ( \ln \frac{\pi(x,y)}{K(x,y)} - 1 \right ) \mathrm{d}\pi(x,y).
\end{equation}
The Sinkhorn distance\footnote{With a slight abuse of language we use the term distance, although the entropy-regularized formulation does not respect the identity of indiscernibles.} is a relaxation of \eqref{eq: wasserstein} where the infimum is over all coupling in $\{\pi\in\Pi(\mu,\nu)|\,\KL(\pi|\mu\times\nu)\leq \xi\}$ for $\xi>-1$. Introduced in \cite{cuturi2013sinkhorn}, the optimization of this distance can be solved by considering the entropy-regularized transport 

\begin{equation}\label{eq:entropy_reg} 
    W_{d,\lambda}^2(\mu, \nu) := \left(\inf_{\pi \in \Pi(\mu, \nu)}  \int_{\gX^2} d(x,y)^2 \mathrm{d}\pi(x, y) - \lambda H(\pi)\right)^{1/2},
\end{equation}
where we define the entropy of a coupling $\pi$ as $H(\pi) := - \int \ln \pi(x, y)\mathrm{d}\pi(x, y)$, and $\lambda>0$. This formulation converges to the Wasserstein distance as $\lambda \to 0$, and can be solved with the Sinkhorn algorithm with complexity of the order $O(n^2 / \epsilon)$ for discrete distributions of size $n$ \cite{cuturi2013sinkhorn}. In the discrete case, the transport matrix $\boldsymbol{\pi}$ admits the form $\text{diag}(\vv)K_\lambda\text{diag}(\vw)$, where $\vv,\vw$ are vectors of size $n$. The Sinkhorn algorithm iteratively updates the vectors as $(\vv,\vw) \leftarrow (\mu./K_\lambda\vw, \nu./K^\prime_\lambda\vv)$, where $K_\lambda:=e^{- d(x,y)^2 / \lambda}$. 

Following \cite{solomon2015convolutional}, using the kernel $K_\lambda$
gives an alternative interpretation of the Sinkhorn distance as
\begin{equation}\label{eq: reg_OT}
    W_{d,\lambda}^2(\mu, \nu) = \lambda^{1/2} \left ( 1 + \min_{\pi \in \Pi(\mu, \nu)} \KL(\pi | K_\lambda) \right )^{1/2}.
\end{equation}

The problem in \eqref{eq:entropy_reg} is strictly convex and continuous yielding a unique minimizer. In the discrete case, this leads to an algorithm for the entropy-regularized Wasserstein distance based on the Sinkhorn algorithm enforcing the marginal constraints on the kernel $K_\lambda$ while minimizing the distance as quantified by $\KL$.

The underlying metric $d(\cdot,\cdot)$ is generally unknown, thus the kernel $K_\lambda$ cannot be evaluated. The authors of~\cite{solomon2015convolutional} proposed to approximate $K_\lambda$ with the heat kernel $\mathcal{H}_t(x,y)$ on $\gX$. According to Varadhan's formula~\cite{varadhan_behavior_1967}, the geodesic distance on a manifold can be recovered from the heat transfer at small timescales as 
\begin{equation}
    d(x,y)^2 = \lim_{t \to 0^+} -4t \ln \mathcal{H}_t(x,y).
\end{equation}
Hence, motivating the use of the heat-geodesic distance $d_{\mathcal{H}}^2(x,y) := -4t \ln \mathcal{H}_t(x,y)$, with associated kernel
    $K_\lambda(x,y) = \mathcal{H}_{\lambda / 4}(x,y).$
Interestingly, Sinkhorn-based methods admit an efficient algorithm to solve the barycenter problem which we present next.

\subsection{Interpolation with discrete support} By constraining the support to a set $\gX$ (or a graph), we can efficiently interpolate between more than two distributions. The barycenter problem \cite{solomon2015convolutional,peyre_computational_2020,cuturi2014fast} generalizes the notion of average between points to an average between distributions. For a set of $m$ distributions $\{\mu_1,\dotsc,\mu_m\}$ supported on $\gX$, the objective is to find a distribution minimizing the average distance
\begin{equation*}
    \mu^*:=\argmin_{\mu\in\gP(\gX)} \sum_{i=1}^m \alpha_i W_d^p(\mu,\mu_i)^p,
\end{equation*}
where $\gP(\gX)$ denotes the space of probability distributions supported on $\gX$, and $\{\alpha_1,\dotsc,\alpha_m\}$ are non-negative weights. Finding the barycenter is a challenging optimization problem, however the barycenter for Sinkhorn-based methods admits an efficient computation. It involves updating $m$ vectors $\vv_i,\vw_i$, which define a transport plan from $\mu_i$ to the barycenter $\mu^*$. 
The support of the barycenter is constrained to $\gX$, for most Sinkhorn-based methods the size of $\gX$ needs to be small for computational reason. Our method does not suffer from such a limitation. Hence, we can consider barycenter with greater expressivity, and interpolate between large sets of distributions.

\subsection{Heat Diffusion on a Graph} Consider an undirected graph $\gG=(V,E)$ with a set $V$ of $n$ vertices and a set of edges $E$, and its weighted adjacency matrix $\mA$ with non-negative edge weights, and the diagonal degree matrix $\mD$, where $\mD_{ii} := \sum_k \mA_{ik}$. We define the combinatorial Laplacian as $\mL:=\mD - \mA$, for any function $f:V\to\mathbb{R}$ we have $(\mL f)(v) = \sum_u a_{u,v} (f(v)-f(u))$. The combinatorial Laplacian is a symmetric positive semi-definite matrix, and has an eigendecomposition $\mL = \Psi \Lambda \Psi^T$ with orthonormal eigenvectors $\Psi$ and diagonal eigenvalue matrix $\Lambda = \mathrm{diag}(\lambda_1, \lambda_2, \ldots, \lambda_n)$, such that $0 \le \lambda_1 \le \lambda_2 \le \cdots \le \lambda_n$. The combinatorial Laplacian is a natural extension of the negative of the Laplacian operator to a graph. For a signal $\vf_0\in\mathbb{R}^n$ on $\gG$, the diffusion of $\vf_0$ on the graph evolves according to the heat equation 
\[
    \frac{d}{dt}\vf(t) + \mL \vf(t) = \bm{0},\,s.t. \quad \vf(0) = \vf_0 \quad t\in\mathbb{R}^+.
\]
The heat kernel solves this ODE, it is defined by the matrix exponential $\mH_t := e^{-t\mL}$. By orthogonality of the eigenvectors of $\mL$, we can write $\mH_t = \Psi e^{-t\Lambda}\Psi^T$ and $\vf(t) = \mH_t\vf_0$. Computing $\mH_t$ by eigendecomposition would require $O(n^3)$ operations. Recall that, for the Sinkhorn algorithm, we are only concerned with the \emph{application} of the heat operator $\mH_t$ on a signal $\vf\in\mathbb{R}^n$. 
For larger diffusion time, the heat kernel converges to its eigenvector associated to the lowest eigenvalues of the Laplacian, hence, intuitively, the heat kernel corresponds to a low-pass filter. In Geodesic Sinkhorn, we use Chebyshev polynomials~\cite{shuman_chebyshev_2011,marcotte2022fast} to approximate the application of the heat operator to a signal. For a short timescale $t$, using the heat kernel accounts for using the geodesic distance as ground distance in the entropy-regularized OT formulation~\eqref{eq:entropy_reg}.

\subsection{Chebyshev Polynomials} Polynomial sequences are often used to approximate functions or operator. With Chebyshev polynomials, we can approximate the application of the matrix exponential $\mH_t = e^{-t\mL}$ to a signal $\vf$ on the graph. An attractive property of Chebyshev polynomials is that the approximation error decays exponentially with the maximum degree $K$. They are defined by the recursive relation $\{T_k\}_{k\in\mathbb{N}}$ with $T_0(y)=0$, $T_1(y)=y$ and $T_k(y) = 2yT_{k-1}(y) - T_{k-2}(y)$ for $k\geq 2$. On $[-1,1]$ these polynomials are orthogonal w.r.t. the weight $(1-y)^{-1/2}$, and can be used to express the operator $\mH_t$. Assuming the largest eigenvalue $\lambda_n\leq2$, we can write
\begin{equation*}
    \mH_t = \frac{b_0}{2} + \sum_{k=1}^\infty b_k T_k(\mL-\bm{Id}),
\end{equation*}
where the $K+1$ scalar coefficient $\{b_k\}$ depend on time and can be evaluated with the Bessel function. The approximation of $\mH_t$ is based on the first $K$ term of the series which we note $p_K(\mL,t)$.
It results in $K$ matrix-vector products which can be efficient since, in general, $\mL$ is a sparse matrix. On a $m$-nearest neighbor graph, this can be $O(Kmn/\lambda)$, where $\lambda$ is a regularization parameter. Chebyshev polynomials admits interesting theoretical properties and are known to converge faster than other polynomials~\cite{marcotte2022fast, huang2020fast}. The choice of the parameter $K$ is related to the number of neighbors or the connectivity of the graph. For small diffusion time, hence only diffusing in a local neighborhood, the approximation is accurate even with a small $K$. As the diffusion time increases, $K$ has to increase in order to consider a larger neighborhood around a node. For OT, we consider small diffusion time, and we found that our results were stable for all $K$ greater than 10. 



\section{Geodesic Sinkhorn Distances}


We define the Geodesic Sinkhorn distance between any signals or distributions on a graph $\gG$ by the entropy-regularized OT with the heat kernel $\mH_t$ on the graph. This construction is also valid between any point cloud datasets. In that case, for $m$ datasets $\{\mathsf{X}_1,\dotsc,\mathsf{X_m}\}$ sampled from a set of distributions $\{\mu_1,\dotsc,\mu_m\}$, we construct a common graph using an affinity kernel on the $m$ datasets and compare two samples by taking the distance between two indicator functions on the graph. We approximate the heat kernel $\mH_t$ with Chebyshev polynomials $p_K(\mL,l)$ of order $K$. In Algorithm~\ref{alg:sinkhorn_geo}, we present the main steps to evaluate the Geodesic Sinkhorn. It is based upon Sinkhorn iterations~\cite{sinkhorn1967concerning, cuturi2013sinkhorn}, where $\oslash$ and $\odot$ denote respectively the elementwise division and multiplication. Note that, as opposed to the usual Sinkhorn algorithm, we never have to store a dense $n\times n$ distance matrix, but only the usually sparse graph Laplacian.

\begin{definition} The Geodesic Sinkhorn distance between two distributions $\mu$, $\nu$ on a graph $\gG$ is
\begin{equation*}
    W_{\mH_t}(\mu,\nu):= 4t^{1/2} \left ( 1 + \min_{\pi \in \Pi(\mu, \nu)} \KL(\pi | \mH_t) \right )^{1/2}.
\end{equation*}
\end{definition}

\begin{algorithm}[tb]
    \caption{Geodesic Sinkhorn}
    \label{alg:sinkhorn_geo}
\begin{algorithmic}
    \STATE {\bfseries Input:} $n \times n$ Laplacian $\mL$, distributions (signals) $\vmu$, $\vnu$ on $\gG$, maximum Chebyshev degree $K$, regularization $\lambda$, $\va$ vertices weights.
    \STATE {\bfseries Output:} $W_{d_{\mH},\lambda}^2(\vmu, \vnu)$
    \STATE \textit{// Initialization}
    \STATE $\vw \leftarrow \vone$
    \STATE \textit{// Sinkhorn iterations}
    \FOR{$j=1,2,\dotsc$}
        \STATE $\vv \leftarrow \vmu \oslash p_K(\mL,t)(\va\odot\vw)$
        \STATE $\vw \leftarrow \vnu \oslash p_K(\mL,t)(\va\odot\vv)$
    \ENDFOR
    \STATE {\bfseries Return:} $4t\,\sum[\va \odot (\vmu \odot \ln \vv + \vnu \odot \ln \vw)]$ 
\end{algorithmic}
\end{algorithm}




In the following proposition, we find the ground distance implicitly used in the optimal transport defined by Geodesic Sinkhorn. We use $\simeq$ for the equivalence relation between distances. 

\begin{proposition}There exists a maximum Chebyshev polynomial degree $K$ such that the ground distance in Geodesic Sinkhorn is equivalent to the one based $\mH_t$
\begin{align*}
   -4t\log  (p_K(\mL,t))_{ij} &\simeq -4t\log (\mH_t)_{ij} 
\end{align*}
In particular, the Wasserstein distances with these ground distances are equivalent. 
\end{proposition}

\begin{proof} Because the approximation error decreases exponentially in $K$ \cite{marcotte2022fast}, we have that for any $\epsilon>0$ sufficiently small there exist $K_i$ such that $(H_t)_{ij}-\epsilon < (p_K)_{ij}<(H_t)_{ij}+\epsilon$. Choose $K$ such that this is true for all vertices $K:=\max\{K_1\dotsc,K_n\}$. We define 
\begin{align*}
    c := \min_{ij}\frac{(\mH_t)_{ij}-\epsilon}{(\mH_t)_{ij}} \text{ and } C := \max_{ij}\frac{(\mH_t)_{ij}+\epsilon}{(\mH_t)_{ij}}
\end{align*}
and we have the equivalence between the distances since
\begin{equation*}
    c(\mH_t)_{ij} \leq (p_K(\mL,t))_{ij} \leq C(\mH_t)_{ij} \text{ for all }ij \in [n],
\end{equation*}
and since the logarithm is a monotonic function. 
\end{proof}

In~\cite{solomon2015convolutional,heitz2021ground}, using the Euler implicit discretization results in a ground cost of the form $-\epsilon \ln ((\mathbf{Id} - \frac{\epsilon}{4K}\mL)^{-K})$, where $\mathbf{Id}$ is the identity matrix, and can be seen as another approximation for the matrix exponential.

The efficiency of Geodesic Sinkhorn improves the notion of barycenter as it is possible to consider much larger graph $\gG$, thus a finer grained support of the barycenter. This leads us to define a novel distance between \emph{families} of distributions.

\begin{definition}
    For two finite families of distributions $\gT$ and $\gC$ supported on $\gG$, we define the barycentric distance between the families $\gT,\gC$ as 
    \[\gamma(\gT, \gC): = W_{\mH_t}(\mu^*_\gT,\mu^*_\gC)\]
    where $\mu^*_\gT,\mu^*_\gC$ are respectively the barycenters of $\gT$ and $\gC$.
\end{definition}
The previous definition is valid for any distances between distributions or barycenters. However, OT barycenters are known to be more informative than others~\cite{cuturi2014fast}. We will further explore this comparison in our experiments. We use it to distinguish between two groups in a medical setting where a set of patients received a treatment (defining the family $\gT$), and another set acts as a control family $\gC$. Following this idea, we define a notion of effect between two families.

\begin{definition}\label{def: bary_effect}
    For two family of distributions $\gT$ and $\gC$ supported on $\gG$, define the Expected Barycenter Effect of $\gT$ as 
    \begin{equation*}
        \tau(\gT):= \mathbb{E}_{\mu_\gT^*}(Y_t) - \mathbb{E}_{\mu_\gC^*}(Y_c),
    \end{equation*}
    where $\mu^*_\gT,\mu^*_\gC$ are respectively the barycenters of $\gT$ and $\gC$, and the features $Y_c\sim\mu_\gC^*$ and $Y_t\sim\mu_\gT^*$.
\end{definition}

Note that we compute the average on the family of distributions instead of the average on their support, hence we evaluate their expectations in a closed form. This definition also extends to a conditional equivalent where families of distribution can be subdivided with discrete covariate variables. When the barycenters are computed with the total variation, this definition is equivalent to the naive Average Treatment Effect(ATE)~\cite{imbens2015causal}; i.e. difference of empirical means.

\vspace{-2mm}
\section{Results}\label{sec:results}
We demonstrate the accuracy and efficiency of the Geodesic Sinkhorn distance on two tasks:
(1) Nearest-Wasserstein-neighbor calculation on simulated data with manifold structure similar to the setup of \cite{tong2021diffusion};
(2) A newly defined Barycentric distance between families of distributions computed to quantify the effect of a treatment on patient-derived organoids. In Appendix~\ref{sec:app_results:time_series}, we present additional results on time series interpolation.

\subsection{Nearest-Wasserstein-neighbor distributions}


In this experiment, we compare our method with Sinkhorn~\cite{cuturi2013sinkhorn}, and LR Sinkhorn~\cite{scetbon2021low}, both algorithms with Euclidean and squared Euclidean ground distance, with DiffusionEMD~\cite{tong_trajectorynet_2020}, and Sinkorn with Euler approximation of the heat filter. We created 15 Gaussian distributions sampled randomly on a swiss roll dataset, and sampled 10k observations from each distribution. We rotated the observations in 10 dimensions. We consider a k-nearest neighbors task on these distributions. We evaluate the methods with the ground truth, since we know the exact geodesic distance on the manifold. In Tab.~\ref{tab:swiss_roll}, we report the average and standard deviation over 10 seeds of the Spearman and Pearson correlations to the ground truth, and the runtime in seconds with and without the computation of the graph. Our method is the most accurate while being much faster than other Sinkhorn-based methods.

\begin{table}[]
    \centering
    \caption{KNN task for 15 distributions, best score highlighted is \first{bold}. Geodesic Sinkhorn is the most accurate, while being faster than other Sinkhorn-based methods.}
    \label{tab:swiss_roll}
\resizebox{\columnwidth}{!}{%
\begin{tabular}{llllll}
\toprule
Method &        SpearmanR &          PearsonR &              P@5 &    Time(s) no graph &             Time(s) \\
\midrule
Diffusion EMD            &   0.62$\pm$0.097 &   0.736$\pm$0.023 &   0.66$\pm$0.072 &     \bf{2.845$\pm$0.135 }&     \bf{7.877$\pm$0.531} \\
Sinkhorn $W1$     &  0.387$\pm$0.044 &   0.523$\pm$0.036 &  0.471$\pm$0.028 &   112.406$\pm$0.206 &   112.406$\pm$0.206 \\
Sinkhorn $W2$   &  0.411$\pm$0.036 &   0.485$\pm$0.027 &  0.492$\pm$0.053 &   133.686$\pm$5.234 &   133.686$\pm$5.234 \\
LR Sinkhorn $W1$   &   -0.31$\pm$0.07 &  -0.131$\pm$0.086 &  0.237$\pm$0.037 &  578.631$\pm$107.82 &  578.631$\pm$107.82 \\
LR Sinkhorn $W2$ &  0.366$\pm$0.048 &   0.379$\pm$0.051 &  0.447$\pm$0.023 &   204.191$\pm$3.656 &   204.191$\pm$3.656 \\
Euler Sinkhorn          &  0.776$\pm$0.061 &   0.718$\pm$0.009 &  0.728$\pm$0.072 &  449.752$\pm$42.985 &  455.059$\pm$43.083 \\
Geodesic Sinkhorn            &  \bf{0.847$\pm$0.023} &   \bf{0.754$\pm$0.016} &  \bf{0.833$\pm$0.034} &    10.176$\pm$1.249 &    16.682$\pm$1.705 \\
\bottomrule
\end{tabular}}
    
\end{table}

\subsection{Barycentric distance}

We test if we can identify a linear treatment effect with the Expected Barycenter Effect~(EBE). In this experiment, we create a control family of distributions $\gC$ of ten standard Gaussian distributions. The treatment group consists of nine Gaussian distributions $\mathcal{N}(5,1)$, and one outlier centered at different means. For each distribution, we sample 500 observations, and reproduce the experiment over ten seeds. In Tab.~\ref{tab:ATE_perturb}, we report the EBE and its standard deviation with the Geodesic Sinkhorn, the Total Variation (TV) distance, and Sinkhorn. Since the TV only compares the mean, it is sensitive to the outlier, whereas our method can identify the true treatment effect.

\begin{table}
    \centering
    \caption{Expected Barycenter Effect~(EBE) with one outlier distribution centered at -60,-30,0, or 5. Comparison using the barycenter from Sinkhorn, total variation, or Geodesic Sinkhorn. Values closer to the real treatment effect of 5 are better.}
\label{tab:ATE_perturb}
\resizebox{\columnwidth}{!}{%
\begin{tabular}{llll}
\toprule
{Outlier} & EBE Geo Sinkhorn   &   EBE Sinkhorn &  EBE TV \\
\midrule
-60  &       \bf{5.016$\pm$0.226}  &  -0.103$\pm$0.005 &  -1.429$\pm$0.144 \\
-30  &       \bf{5.053$\pm$0.196} &   0.355$\pm$0.049 &   1.571$\pm$0.144 \\
0  &       4.917$\pm$0.315 &         \textbf{4.954$\pm$0.157} &   4.571$\pm$0.144 \\
\midrule
No outlier  &       5.059$\pm$0.159  &  \bf{5.054$\pm$0.16} &   5.071$\pm$0.144 \\
\bottomrule
\end{tabular}
}
\end{table}

\subsection{Single-cell signaling data}

 We use single-cell signaling data produced by mass cytometry (MC) for a screening study to compare the treatment effect of different chemotherapies on 10 colorectal cancer (CRC) patient-derived organoids (PDOs)~\cite{zapatero2022pdos}. These PDOs can be grouped into chemoresistant PDOs, that show little-to-no effect when treated with chemotherapies; and chemosensitive PDOs, that present strong shifts in their phenotypes upon treatment. The observations include single-cell data information on the cell cycle and signaling changes upon treatment of PDOs with different CRC treatments at a range of concentrations. 
In Fig.~\ref{fig:comparison_treat_conc}, we present the barycentric distances matrices between treatments a) and between four concentrations of treatment SN-38 (S) b). In both cases, the control groups corresponds to AH and DMSO, the two rightmost columns. We compare the distance matrices between Sinkhorn (left) and our method (right). Our method provide a finer distinction between treatments (Fig. \ref{fig:comparison_treat_conc} top) and concentrations (Fig. \ref{fig:comparison_treat_conc} bottom), especially for the chemosensitive group. As observed in~\cite{zapatero2022pdos}, chemosensitive PDOs show little-to-no response to lower concentrations of SN-38 (S1), but their phenotype shifts very strongly  upon treatment with higher concentrations (S2, S3, and S4) (Fig.~\ref{fig:comparison_treat_conc}~b). When comparing combinations of different treatments (Fig.~\ref{fig:comparison_treat_conc}~a), Geodesic Sinkhorn better resolves the difference between SN-38 (S) alone and in combination with Cetuximab (C), showing that S is the main agent creating the treatment effect and the combination with C does not resolve in a synergistic effect~\cite{zapatero2022pdos}. Note that we only consider the relative magnitude of the distances, since the two algorithm use different ground distances. 


\begin{figure}[ht]
    \centering
    \includegraphics[width=1\linewidth]{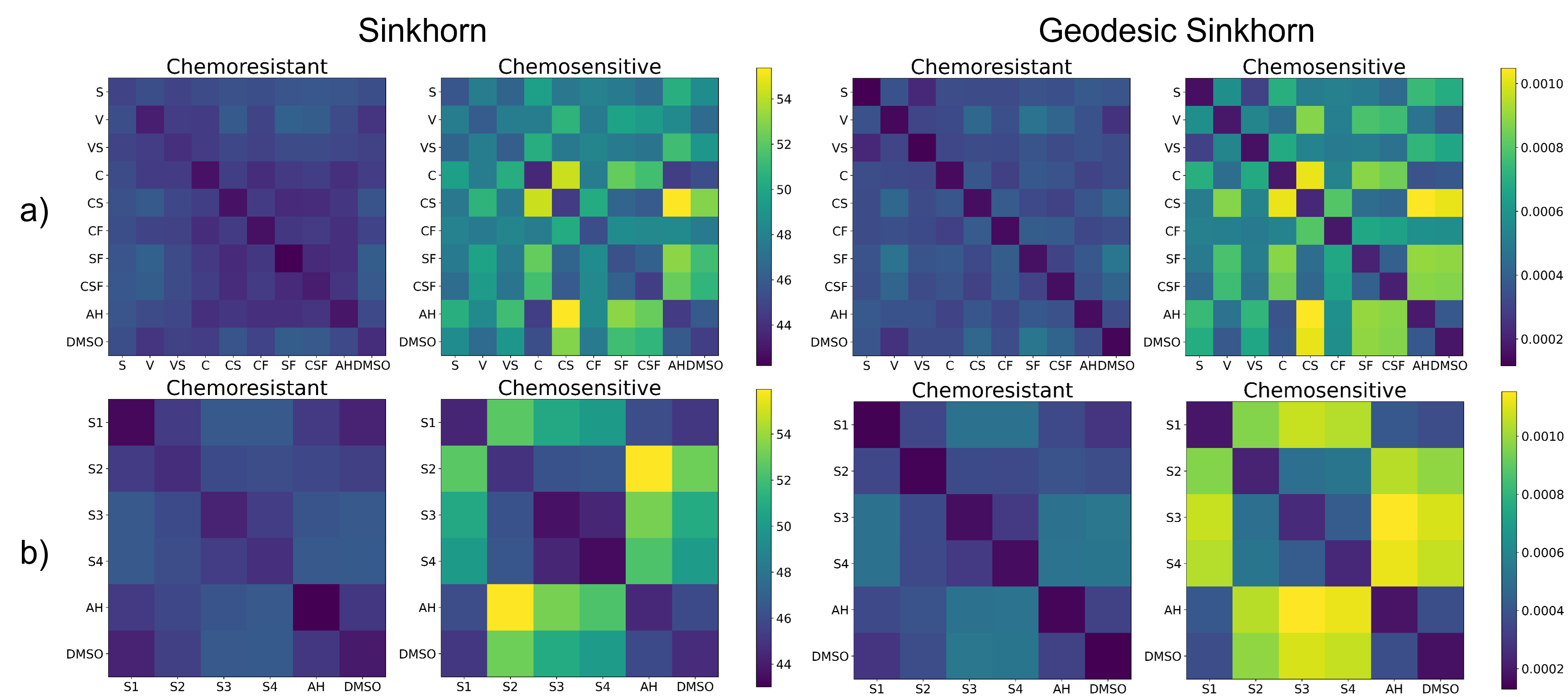}
    \caption{a) Barycentric distances matrices for the Sinkhorn algorithm (left) and our method Geodesic Sinkhorn (right). b)  Barycentric distances matrices between doses of treatment SN-38, for four concentrations S1 S2 S3 S4. Control groups correspond to AH and DSMO. Geodesic Sinkhorn provides a clearer distinction between treatments, and doses.}
    \label{fig:comparison_treat_conc}
\end{figure}


\section{Conclusion}

In this work, we considered the use of OT for graphs and large datasets in high dimensions potentially sampled from a lower dimensional manifold. We proposed Geodesic Sinkhorn, a fast implementation of the Sinkhorn algorithm using the graph Laplacian and Chebyshev polynomials. Our method is well adapted for large and high dimensional datasets as it is defined with a geodesic ground distance, which takes into account the underlying geometry of the data, and requires less computation time and less memory. On a synthetic dataset, we showed that Geodesic Sinkhorn is much faster than other Sinkhorn-based methods while being more accurate. With the Wasserstein barycenter, we defined the barycentric distance to compare entire families of distributions, and the expected barycenter effect, then applied both methods to a large PDO drug screen dataset.



\printbibliography

\clearpage
\appendix
\onecolumn

\section{Appendix}
\begin{refsection}
\subsection{Time series Interpolation}\label{sec:app_results:time_series}

To evaluate Geodesic Sinkhorn's performance on inferring dynamics, we test its performance on a task for time series interpolation. In this setting the most used datasets are Embryoid Body~\cite{moon_visualizing_2019}, and WOT~\cite{schiebinger_optimal-transport_2019}. We curated ten scRNA-seq datasets; WOT~\cite{schiebinger_optimal-transport_2019}, Clark~\cite{clark_single-cell_2019}, Embryoid Body~\cite{moon_visualizing_2019}, and seven more from the 2022 multimodel single-cell integration challenge\footnote{\url{https://www.kaggle.com/competitions/open-problems-multimodal/}} to test our method. The observations are the gene expression of single cells from a distribution evolving through time. The Waddington-OT dataset (WOT) has 38 timpoints of a developing stem cell population over 18 days collected roughly every 6-12 hours. This is the most densely sampled dataset in time. The Embryoid Body dataset is a single-cell RNA seq dataset of developing embryoid bodies from 0-30 days with 5 datasets collected over time. The Clark dataset contains 12 samples over 9 unique timepoints of a developing mouse retina. Finally, the NeurIPS 2022 data contains four donors with single-cell transcriptomic data collected over 4 timepoints for 3 donors with 10X-multiome and 4 combined cite-seq (in the publically released training data) leading to 7 additional time series with 4 timepoints.

The goal is to interpolate the distribution between two timepoints. The number of timepoints in each dataset range from 4 to 40. Here for a dataset with $T$ single-cell distributions $\{\mu_1, \mu_2, \ldots, \mu_T\}$ over time for $t \in [2 \ldots T - 1]$ we compute the exact Euclidean 2-Wasserstein distance between the interpolated distribution $\hat{\mu}_t$ at time $t$ and the ground truth distribution $\mu_t$, $W_2(\hat{\mu}_t, \mu_t)$. Since we interpolate between two distributions we used the McCann interpolant as its support is $\mathbb{R}^d$. We compare our Geodesic Sinkhorn interpolation with either the Sinkhorn Mccann interpolant with Euclidean ground distance ($L^2$ Sinkhorn), Sinkhorn Euler Mccann with the Euler heat approximation in Tab.~\ref{tab:time_interp}. Sinkhorn with Euler approximation ran out of memory on the Clark and WOT datasets.We see that across all 10 the Geodesic Sinkhorn interpolation with the Mccann interpolant outperforms all other methods, hence showcasing the importance of the heat-geodesic distance and our kernel approximation. We also compare Sinkhron Euler and Geodesic Sinkhorn for different nearest neighbors graphs in Tab.~\ref{tab:nn_50} and Tab.~\ref{tab:nn_100}, where Geodesic Sinkhorn outperforms Sinkhorn Euler on most datasets.

\begin{table}[h!]
    \centering
    \caption{Time series interpolation task comparing mean and standard deviation $(\mu \pm \sigma)$ across 5 seeds the 2-Wasserstein metric averaged across time $(1/T-2) \sum_{t \in [2 \ldots T - 1]} W_2(\hat{\mu}_t, \mu_t)$ for 10 single-cell timeseries datasets. Sinkhron Euler ran out of memory on two datasets. Lower is better, best performance on each dataset is \textbf{bold}.}
\begin{tabular}{lrrr}
\toprule
{Dataset} &  {$L^2$ Sinkhorn} &  {Sinkhron Euler} &  {Geo Sinkhorn (ours)} \\
\midrule
Cite Donor0 &   48.545 $\pm$ 0.057 &   46.254 $\pm$ 3.192 &   \textbf{44.440 $\pm$ 0.108} \\
Cite Donor1 &   48.220 $\pm$ 0.055 &   45.897 $\pm$ 3.254 &   \textbf{44.165 $\pm$ 0.103} \\
Cite Donor2 &   50.281 $\pm$ 0.016 &   47.773 $\pm$ 3.958 &   \textbf{45.673 $\pm$ 0.092} \\
Cite Donor3 &   49.339 $\pm$ 0.081 &  46.565 $\pm$ 3.553 &   \textbf{45.022 $\pm$ 0.146} \\
Clark &   13.500 $\pm$ 0.003 &   --   & \textbf{13.288 $\pm$ 0.008} \\
EB &   12.415 $\pm$ 0.008 &   12.298 $\pm$ 0.140 &   \textbf{12.133 $\pm$ 0.011} \\
Multiome Donor0 &   56.648 $\pm$ 0.048 &   55.373 $\pm$ 7.234 &   \textbf{53.431 $\pm$ 0.077} \\
Multiome Donor1 &   54.028 $\pm$ 0.126 &   52.396 $\pm$ 4.394 &   \textbf{50.238 $\pm$ 0.022} \\
Multiome Donor2 &   58.798 $\pm$ 0.155 &   57.182 $\pm$ 5.511 &   \textbf{55.041 $\pm$ 0.058} \\
WOT &   8.096 $\pm$ 0.003 &   -- &  \textbf{7.397 $\pm$ 0.106} \\
\bottomrule
\end{tabular}%
\label{tab:time_interp}
\end{table}

\begin{table}[]
    \centering
\begin{tabular}{lrr}
\toprule
Dataset &  Sinkhorn Euler &  Geodesic Sinkhorn \\
\midrule
Cite Donor0     &                                  48.507 &                          \bf{46.850} \\
Cite Donor1     &                                  48.207 &                          \bf{46.883} \\
Cite Donor2     &                                  50.383 &                          \bf{49.176} \\
Cite Donor3     &                                  49.210 &                          \bf{47.646} \\
Clark           &                                  13.506 &                          \bf{13.378} \\
EB              &                                  12.409 &                          \bf{12.394} \\
Multiome Donor0 &                                  56.676 &                          \bf{55.095} \\
Multiome Donor1 &                                  54.028 &                          \bf{53.952} \\
Multiome Donor2 &                                  58.821 &                          \bf{57.187} \\
WOT &                                   \bf{8.070} &                           8.279 \\
\bottomrule
\end{tabular}
    \caption{Time series interpolation task comparing mean of the 2-Wasserstein for a KNN graph with $50$ neighbors. Lower is better, best score is \textbf{bold}.}
    \label{tab:nn_50}
\end{table}

\begin{table}[]
    \centering
\begin{tabular}{lrr}
\toprule
Dataset &  Sinkhorn Euler &  Geodesic Sinkhorn \\
\midrule
Cite Donor0     &                                  48.573 &                          \bf{47.102} \\
Cite Donor1     &                                  48.197 &                          \bf{47.373} \\
Cite Donor2     &                                  50.298 &                          \bf{49.517} \\
Cite Donor3     &                                  49.290 &                          \bf{48.433} \\
Clark   &                                  13.500 &                          \bf{13.396} \\
EB              &                                  \bf{12.415} &                          12.416 \\
Multiome Donor0 &                                  56.708 &                          \bf{55.562} \\
Multiome Donor1 &                                  54.063 &                          \bf{54.016} \\
Multiome Donor2 &                                  58.802 &                          \bf{57.732} \\
WOT &                                   \bf{8.096} &                           8.253 \\
\bottomrule
\end{tabular}
    \caption{Time series interpolation task comparing mean of the 2-Wasserstein for a KNN graph with $100$ neighbors. Lower is better, best score is \textbf{bold}.}
    \label{tab:nn_100}
\end{table}

\clearpage
\printbibliography[heading=subbibliography]
\end{refsection}

\end{document}